\documentclass{article}

\usepackage{times}
\usepackage{graphicx} % more modern
\usepackage{subfigure} 

\usepackage{natbib}

\usepackage[algo2e,inoutnumbered,linesnumbered,algoruled,vlined]{algorithm2e}

\usepackage{hyperref}

\usepackage[accepted]{icml2016}

\usepackage{amsmath}
\usepackage{comment}
\usepackage{amssymb}
\usepackage{booktabs}
\usepackage{graphicx}
\usepackage{subfigure}
\usepackage{bbold}

\usepackage{enumitem}

\usepackage{amsthm}

\icmltitlerunning{Stochastic Quasi-Newton Langevin Monte Carlo}

\DeclareMathOperator*{\argmax}{arg\max}

\usepackage{ dsfont }
\usepackage{wrapfig}

\usepackage[utf8]{inputenc}
\usepackage[T1]{fontenc}
\usepackage{amssymb}
\newcounter{cmtcounter}
\newtheoremstyle{exampstyle}
  {\topsep} % Space above
  {\topsep} % Space below
  {} % Body font
  {} % Indent amount
  {\bfseries} % Theorem head font
  {.} % Punctuation after theorem head
  {.5em} % Space after theorem head
  {} % Theorem head spec (can be left empty, meaning `normal')
\theoremstyle{exampstyle}

\newtheorem{thm}{Theorem}

\theoremstyle{exampstyle} \newtheorem{prop}{Proposition}

\makeatletter
\renewenvironment{proof}[1][\proofname]{\par
  \vspace{-\topsep}% remove the space after the theorem
  \pushQED{\qed}%
  \normalfont
  \topsep0pt \partopsep0pt % no space before
  \trivlist
  \item[\hskip\labelsep
        \itshape
    #1\@addpunct{.}]\ignorespaces
}{%
  \popQED\endtrivlist\@endpefalse
  \addvspace{6pt plus 6pt} % some space after
}
\makeatother

\begin{document} 

\twocolumn[
\icmltitle{Stochastic Quasi-Newton Langevin Monte Carlo }

% unless you've provided the [accepted] option to the icml2015
\icmlauthor{Umut \c Sim\c sekli$^{\text{1}}$}{umut.simsekli@telecom-paristech.fr}
\icmlauthor{Roland Badeau$^{\text{1}}$}{roland.badeau@telecom-paristech.fr}
\icmlauthor{A. Taylan Cemgil$^{\text{2}}$}{taylan.cemgil@boun.edu.tr}
\icmlauthor{Ga\"{e}l Richard$^{\text{1}}$}{gael.richard@telecom-paristech.fr}
\icmladdress{1: LTCI, CNRS, T\'{e}l\'{e}com ParisTech, Universit\'{e} Paris-Saclay, 75013, Paris, France\\
2: Department of Computer Engineering, Bo\u{g}azi\c ci University, 34342, Bebek, \.{I}stanbul, Turkey}
% \icmladdress{Dept. of Computer Engineering, Bo\u{g}azi\c ci University, 34342, Bebek, \.{I}stanbul, Turkey}
% \icmlauthor{Your CoAuthor's Name}{email@coauthordomain.edu}
% \icmladdress{Their Fantastic Institute,
            % 27182 Exp St., Toronto, ON M6H 2T1 CANADA}

\icmlkeywords{Markov Chain Monte Carlo, Stochastic Quasi Newton, Matrix Factorization}

\vskip 0.3in
]

\begin{abstract} 
Recently, Stochastic Gradient Markov Chain Monte Carlo (SG-MCMC) methods have been proposed for scaling up Monte Carlo computations to large data problems. Whilst these approaches have proven useful in many applications, vanilla SG-MCMC might suffer from poor mixing rates when random variables exhibit strong couplings under the target densities or big scale differences. In this study, we propose a novel SG-MCMC method that takes the local geometry into account by using ideas from Quasi-Newton optimization methods. These second order methods directly approximate the inverse Hessian by using a limited history of samples and their gradients. Our method uses dense approximations of the inverse Hessian while keeping the time and memory complexities linear with the dimension of the problem. We provide a formal theoretical analysis where we show that the proposed method is asymptotically unbiased and consistent with the posterior expectations. We illustrate the effectiveness of the approach on both synthetic and real datasets. Our experiments on two challenging applications show that our method achieves fast convergence rates similar to Riemannian approaches while at the same time having low computational requirements similar to diagonal preconditioning approaches. 
\end{abstract}

\section{Introduction}
\label{sec:intro}

Markov Chain Monte Carlo (MCMC) methods are one of the most important family of algorithms in Bayesian machine learning.
These methods lie in the core of various applications such as the estimation of Bayesian predictive densities and Bayesian model selection. They also provide important advantages over optimization-based point estimation methods, such as having better predictive accuracy and being more robust to over-fitting \cite{ChenICML2014,Ahn15}. Despite their well-known benefits, during the last decade, conventional MCMC methods have lost their charm since they are often criticized as being computationally very demanding. Indeed, classical approaches based on batch Metropolis Hastings would require passing over the whole data set at each iteration and the acceptance-rejection test makes the methods even more impractical for modern data sets. 

Recently, alternative approaches have been proposed to scale-up MCMC inference to large-scale regime. An important attempt was made by \citet{WelTeh2011a}, where the authors combined the ideas from Langevin Monte Carlo (LMC) \cite{RosskyDollFriedman1978,Roberts03,neal2010} and stochastic gradient descent (SGD) \cite{robbins1951,kushner}, and developed a scalable MCMC framework referred to as stochastic gradient Langevin dynamics (SGLD). Unlike conventional batch MCMC methods, SGLD uses subsamples of the data per iteration similar to SGD. With this manner, SGLD is able to scale up to large datasets while at the same time being a valid MCMC method that forms a Markov chain asymptotically sampling from the target density. Several extensions of SGLD have been proposed \cite{AhnKorWel2012,PatTeh2013a,AhnShaWel2014,ChenICML2014,DingFBCSN14,ma2015complete,chen2015convergence,shang2015covariance,li2015preconditioned}, which are coined under the term Stochastic Gradient MCMC (SG-MCMC).

One criticism that is often directed at SGLD is that it suffers from poor mixing rates when the target densities exhibit strong couplings and scale differences across dimensions. This problem is caused by the fact that SGLD is based on an isotropic Langevin diffusion, which implicitly assumes that different components of the latent variable are uncorrelated and have the same scale, a situation which is hardly encountered in practical applications.

In order to be able to generate samples from non-isotropic target densities in an efficient manner, SGLD has been extended in several ways. The common theme in these methods is to incorporate the geometry of the target density to the sampling algorithm.
Towards this direction, a first attempt was made by \cite{AhnKorWel2012}, where the authors proposed Stochastic Gradient Fisher Scoring (SGFS), a preconditioning schema that incorporates the curvature via Fisher scoring. Later, \citet{PatTeh2013a} presented Stochastic Gradient Riemannian Langevin Dynamics (SGRLD), where they defined the sampler on the Riemannian manifold by borrowing ideas from \cite{girolami2011riemann}. SGRLD incorporates the local curvature information through the expected Fisher information matrix (FIM), which defines a natural Riemannian metric tensor for probability distributions. SGRLD has shown significant performance improvements over SGLD; however, its computational complexity is very intensive since it requires storing and inverting huge matrices, computing Cholesky factors, and expensive matrix products for generating each sample. It further requires computing the third order derivatives and the expected FIM to be analytically tractable, which limit the applicability of the method to a narrow variety of statistical models \cite{calderhead2012sparse}.

In a very recent study, \citet{li2015preconditioned} proposed Preconditioned SGLD (PSGLD), that aims to handle the scale differences in the target density by making use of an adaptive preconditioning matrix that is chosen to be diagonal for computational purposes. Even though this method is computationally less intensive than SGRLD, it still cannot handle highly correlated target densities since the preconditioning matrix is forced to be diagonal. Besides, in order to reduce the computational burden, the authors discard a correction term in practical applications, which introduces permanent bias that does not vanish asymptotically.

In this study, we propose Hessian Approximated MCMC (HAMCMC), a novel SG-MCMC method that computes the local curvature of the target density in an accurate yet computationally efficient manner. Our method is built up on similar ideas from the Riemannian approaches; however, instead of using the expected FIM, we consider the local Hessian of the negative log posterior, whose expectation coincides with the expected FIM. Instead of computing the local Hessian and inverting it for each sample, we borrow ideas from Quasi-Newton optimization methods and directly approximate the inverse Hessian by using a limited history of samples and their stochastic gradients. By this means, HAMCMC is (i) computationally more efficient than SGRLD since its time and memory complexities are linear with the dimension of the latent variable, (ii) it can be applied to a wide variety of statistical models without requiring the expected FIM to be analytically tractable, and (iii) it is more powerful than diagonal preconditioning approaches such as PSGLD, since it uses dense approximations of the inverse Hessian, hence is able to deal with correlations along with scale differences.

We provide rigorous theoretical analysis for HAMCMC, where we show that HAMCMC is asymptotically unbiased and consistent with posterior expectations. We evaluate HAMCMC on a synthetic and two challenging real datasets. In particular, we apply HAMCMC on two different matrix factorization problems and we evaluate its performance on a speech enhancement and a distributed link prediction application. Our experiments demonstrate that HAMCMC achieves fast convergence rates similar to SGRLD while at the same time having low computational requirements similar to SGLD and PSGLD.   

We note that SGLD has also been extended by making use of higher order dynamics \cite{ChenICML2014,DingFBCSN14,ma2015complete,shang2015covariance}. These extensions are rather orthogonal to our contributions and HAMCMC can be easily extended in the same directions.

\section{Preliminaries}

Let $\theta$ be a random variable in $\mathbb{R}^D$. We aim to sample from the posterior distribution of $\theta$, given as $p(\theta|x) \propto \exp(-U(\theta))$, where $U(\theta)$ is often called the \emph{potential energy} and defined as $U(\theta) = - [\log p(x|\theta) + \log p(\theta)]$. Here, $x \triangleq \{x_n\}_{n=1}^N$ is a set of observed data points, and each $x_n \in \mathbb{R}^P$. We assume that the data instances are independent and identically distributed, so that we have:
\begin{align} 
U(\theta) &= -[\log p(\theta) + \sum_{n=1}^N \log p(x_n |\theta)] .
\end{align}
We define an unbiased estimator of $U(\theta)$ as follows:
\begin{align}
\tilde{U}(\theta) &= -[\log p(\theta) + \frac{N}{N_\Omega} \sum_{n \in \Omega} \log p(x_n |\theta)] 
\end{align}
where $\Omega \subset \{1,\dots,N\}$ is a random data subsample that is drawn with replacement, $N_\Omega = |\Omega|$ is the number of elements in $\Omega$. We will occasionally use $\tilde{U}_\Omega(\theta)$ for referring to $\tilde{U}(\theta)$ when computed on a specific subsample $\Omega$. 

\subsection{Stochastic Gradient Langevin Dynamics}

By combining ideas from LMC and SGD, \citet{WelTeh2011a} presented SGLD that asymptotically generates a sample $\theta_t$ from the posterior distribution by iteratively applying the following update equation: %\cite{WelTeh2011a}:
\begin{align} 
\theta_t = \theta_{t-1} - \epsilon_t \nabla \tilde{U}(\theta_{t-1}) + \eta_{t} \label{eqn:sgld}
\end{align} 
where $\epsilon_t$ is the step-size, and $\eta_t$ is Gaussian noise:
%\begin{align}
$\eta_t \sim {\cal N}(0, 2\epsilon_t I)$, and ${I}$ stands for the identity matrix. 

This method can be seen as the Euler discretization of the Langevin dynamics that is described by the following stochastic differential equation (SDE):
\begin{align}
d \theta_t = - \nabla {U}(\theta_t) dt + \sqrt{2} dW_t
\end{align}
where $W_t$ is the standard Brownian motion. In the simulations $\nabla U(\theta)$ is replaced by the stochastic gradient $\nabla \tilde{U}(\theta)$, which is an unbiased estimator of $\nabla U(\theta)$. Convergence analysis of SGLD has been studied in \cite{icml2014c2_satoa14,TehThiVol2014a}. Since $W_t$ is isotropic, SGLD often suffers from poor mixing rates when the target density is highly correlated or contains big scale differences.

\subsection{The L-BFGS algorithm}

In this study, we consider a \emph{limited-memory} Quasi-Newton (QN) method, namely the L-BFGS algorithm \cite{nocedal} that iteratively applies the following equation in order to find the maximum a-posteriori estimate: 
\begin{align} 
\theta_t = \theta_{t-1} - \epsilon_t H_t \nabla {U}(\theta_{t-1}) \label{eqn:qn}
\end{align}
where $H_t$ is an approximation to the inverse Hessian at $\theta_{t-1}$. The L-BFGS algorithm directly approximates the inverse of the Hessian by using the $M$ most recent values of the past iterates. At each iteration, the inverse Hessian is approximated by applying the following recursion: (using $H_{t} = H_{t}^{M-1}$)
\begin{align}
% H_{t}^m = (I - \frac{y_\tau s_\tau^\top}{s_\tau^\top y_\tau}) H_{t}^{m-1} (I - \frac{s_\tau y_\tau^\top}{s_\tau^\top y_\tau}) + s_\tau s_\tau^\top
H_{t}^m = (I - \frac{s_\tau y_\tau^\top}{y_\tau^\top s_\tau}) H_{t}^{m-1} (I - \frac{y_\tau s_\tau^\top}{y_\tau^\top s_\tau}) + \frac{s_\tau s_\tau^\top}{y_\tau^\top s_\tau} \label{eqn:lbfgs}
\end{align}
where $s_t \triangleq \theta_{t} - \theta_{t-1}$, $y_t \triangleq \nabla {U}(\theta_{t}) - \nabla {U}(\theta_{t-1})$, and $\tau = t-M+m$ and the initial approximation is chosen as $H_t^1 = \gamma I$ for some $\gamma>0$. The matrix-vector product $H_t \nabla U(\theta_{t-1})$ is often implemented either by using the \emph{two-loop recursion} \cite{nocedal} or by using the \emph{compact form} \cite{byrd1994representations}, which both have linear time complexity ${\cal O}(MD)$. Besides, since L-BFGS needs to store only the latest $M-1$ values of $s_t$ and $y_t$, the space complexity is also ${\cal O}(MD)$, as opposed to classical QN methods that have quadratic memory complexity.

Surprisingly, QN methods have not attracted much attention from the MCMC literature. \citet{zhang2011quasi} combined classical Hamiltonian Monte Carlo with L-BFGS for achieving better mixing rates for small- and medium-sized problems. This method considers a batch scenario with full gradient computations which are then appended with costly Metropolis acceptance steps. Recently, \citet{dahlin2015quasi} developed a particle Metropolis-Hastings schema based on \cite{zhang2011quasi}.

In this study, we consider a stochastic QN (SQN) framework for MCMC. The main difference in these approaches is that they use stochastic gradients $\nabla \tilde{U}(\theta)$ in the L-BFGS computations. These methods have been shown to provide better scalability properties than QN methods, and more favorable convergence properties than SGD-based approaches. However, it has been shown that straightforward extensions of L-BFGS to the stochastic case would fail in practice and therefore special attention is required \cite{schraudolph2007stochastic,byrd:2014}. %moritz2015linearly

\section{Stochastic Quasi-Newton LMC}

Recent studies have shown that Langevin and Hamiltonian Monte Carlo methods have strong connections with optimization techniques. Therefore, one might hope for developing stronger MCMC algorithms by borrowing ideas from the optimization literature \cite{qiminkahessian,zhang2011quasi,buithanh2012scaled,pereyra2013proximal,nonsmoothhmc,bubeck2015finite,li2015preconditioned}. However, incorporating ideas from optimization methods to an MCMC framework requires careful design and straightforward extensions often do not yield proper algorithms \cite{ChenICML2014,ma2015complete}. In this section, we will first show that a na\"{i}ve way of developing an SG-MCMC algorithm based on L-BFGS would not target the correct distribution. Afterwards, we will present our proposed algorithm HAMCMC and show that it targets the correct distribution and it is asymptotically consistent with posterior expectations.

\subsection{A Na\"{i}ve Algorithm}

A direct way of using SQN ideas in an LMC framework would be achieved by considering $H_t$ as a preconditioning matrix in an SGLD context \cite{WelTeh2011a,AhnKorWel2012} by combining Eq.~\ref{eqn:sgld} and Eq.~\ref{eqn:qn}, which would yield the following update equation:
\begin{align}
\theta_t = \theta_{t-1} - \epsilon_t H_t(\theta_{t-M:t-1}) \nabla \tilde{U}(\theta_{t-1}) + \eta_t \label{eqn:naive}
\end{align}
where $H_t(\cdot)$ is the approximate inverse Hessian computed via L-BFGS, $M$ denotes the memory size in L-BFGS and $\eta_t \sim {\cal N}\bigl(0, 2 \epsilon_t H_t(\theta_{t-M:t-1})\bigr)$. Here, we use a slightly different notation and denote the L-BFGS approximation via $H_t(\theta_{t-M:t-1})$ instead of $H_t$ in order to explicitly illustrate the samples that are used in the L-BFGS calculations.

Despite the fact that such an approach would introduce several challenges, which will be described in the next section, for now let us assume computing Eq.~\ref{eqn:naive} is feasible. Even so, this approach does not result in a proper MCMC schema.
\begin{prop}
The Markov chain described in Eq.~\ref{eqn:naive} does not target the posterior distribution $p(\theta|x)$ unless $\sum_{j=1}^D\frac{\partial}{\partial \theta_j} H_{ij}(\theta) = 0$ for all $i \in \{1,\dots,D\}$.
\end{prop}
\begin{proof}
Eq.~\ref{eqn:naive} can be formulated as a discretization of a continuous-time diffusion that is given as follows:
\begin{align}
d \theta_t = -H(\theta_{t}) \nabla {U}(\theta_{t}) dt + \sqrt{2 H(\theta_{t}) } dW_t  \label{eqn:naivesde}
\end{align}
where $\nabla {U}(\theta_{t})$ is approximated by $\nabla \tilde{U}(\theta_{t})$ in simulations. Theorems 1 and 2 of \cite{ma2015complete} together show that SDEs with state dependent volatilities leave the posterior distribution invariant only if the drift term contains an additional `correction' term, $\Gamma(\theta_{t})$, defined as follows:\footnote{The correction term presented in \cite{girolami2011riemann} and \cite{PatTeh2013a} has a different form than Eq.~\ref{eqn:corr_term}. \citet{xifara2014langevin} later showed that this term corresponded to a non-Lebesgue measure and Eq.~\ref{eqn:corr_term} should be used instead for the Lebesgue measure.} 
\begin{align}
\Gamma_i (\theta) = \sum_{j = 1}^D \frac{\partial}{\partial \theta_j} H_{ij}(\theta). \label{eqn:corr_term}
\end{align}
Due to the hypothesis, Eq.~\ref{eqn:naivesde} clearly does not target the posterior distribution, therefore Eq.~\ref{eqn:naive} does not target the posterior distribution either.% unless $\Gamma(\theta) = 0, \forall \theta$ . 
\end{proof}

In fact, this result is neither surprising nor new. It has been shown that in order to ensure targeting the correct distribution, we should instead consider the following SDE \cite{girolami2011riemann,xifara2014langevin,PatTeh2013a,ma2015complete}:
\begin{align*}
d \theta_t = -[H(\theta_{t}) \nabla {U}(\theta_{t}) + \Gamma(\theta_{t})] dt + \sqrt{2 H(\theta_{t}) } dW_t
\end{align*}
where $\Gamma(\theta_t)$ must be taken into account when generating the samples. If we choose $H(\theta_t) = \mathds{E}[\nabla^2 U(\theta_t)]^{-1}$ (i.e., the inverse of the expected FIM), discretize this SDE with the Euler scheme, and approximate $\nabla {U}(\theta_{t})$ with $\nabla \tilde{U}(\theta_{t})$, we obtain SGRLD. Similarly, we obtain PSGLD if $H(\theta_t)$ has the following form:
\begin{align*}
H(\theta_t) &= \text{diag}( 1 / (\lambda + \sqrt{v(\theta_t)} )  ) \\
v(\theta_t) &= \alpha v(\theta_{t-1}) + (1-\alpha) \bar{g}_{\Omega_{t-1}}(\theta_{t-1}) \circ \bar{g}_{\Omega_{t-1}}(\theta_{t-1}) 
\end{align*}
where $\bar{g}_{\Omega}(\theta) \triangleq (1/N_{\Omega}) \sum_{n\in\Omega}\nabla \log p(x_n|\theta)$, the symbols $\cdot/\cdot$ and $\cdot\circ\cdot$ respectively denote element-wise division and multiplication, and $\alpha \in [0,1]$ and $\lambda \in \mathds{R}_+$ are the parameters to be tuned.

The correction term $\Gamma(\theta_t)$ does not vanish except for a very limited variety of models and computing this term requires the computation of the second derivatives of $\tilde{U}(\theta_t)$ in our case (whereas it requires computing the third derivatives in the Riemannian methods). This conflicts with our motivation of using QN methods for avoiding the Hessian computations in the first place. 

In classical Metropolis-Hastings settings, where Langevin dynamics is only used in the proposal distribution, the correction term $\Gamma(\theta_t)$ can be  discarded without violating the convergence properties thanks to the acceptance-rejection step \cite{qiminkahessian,girolami2011riemann,zhang2011quasi,calderhead2012sparse}. However, such an approach would result in slower convergence \cite{girolami2011riemann}. On the other hand, in SG-MCMC settings which do not include an acceptance-rejection step, this term can be problematic and should be handled with special attention.

In \cite{AhnKorWel2012}, the authors discarded the correction term in a heuristic manner. Recently, \citet{li2015preconditioned} showed that discarding the correction term induces permanent bias which deprives the methods of their asymptotic consistency and unbiasedness. In the sequel, we will show that the computation of $\Gamma(\theta_t)$ can be avoided without inducing permanent bias by using a special construction that exploits the limited memory structure of L-BFGS.

\subsection{Hessian Approximated MCMC}

In this section, we present our proposed method HAMCMC. HAMCMC applies the following update rules for generating samples from the posterior distribution:
\begin{align}
\theta_t = \theta_{t-M} - \epsilon_t  H_t \Bigl( {\theta}_{t-2M + 1:t-1}^{\neg (t-M)} \Bigr) \nabla \tilde{U}(\theta_{t-M}) + \eta_t \label{eqn:hamcmc}
\end{align}
where ${\theta}_{a:b}^{\neg c} \equiv (\theta_{a:b} \setminus \theta_{c})$, $\eta_t \sim {\cal N}\bigl(0, 2 \epsilon_t H_t(\cdot)\bigr)$, and we assume that the initial $2M+1$ samples, i.e. $\theta_0, \dots, \theta_{2M}$ are already provided. 

As opposed to the na\"{i}ve approach, HAMCMC generates the sample $\theta_t$ based on $\theta_{t-M}$ and it uses a history samples $\theta_{t-2M+1},\dots,\theta_{t-M-1},\theta_{t-M+1},\dots,\theta_{t-1}$ in the L-BFGS computations. Here, we define the displacements and the gradient differences in a slightly different manner from usual L-BFGS, given as follows: $s_t = \theta_t - \theta_{t-M}$ and $y_t = \nabla \tilde{U}(\theta_t)-\nabla \tilde{U}(\theta_{t-M})$. Therefore, $M$ must be chosen at least $2$ or higher. HAMCMC requires to store only the latest $M$ samples and the latest $M-1$ memory components, i.e. $\theta_{t-M:t-1}$, $s_{t-M+1:t-1}$, and $y_{t-M+1:t-1}$, resulting in linear memory complexity.

An important property of HAMCMC is that the correction term $\Gamma(\theta)$ vanishes due to the construction of our algorithm, which will be formally proven in the next section. Informally, since $H_t$ is independent of $\theta_{t-M}$ conditioned on ${\theta}_{t-2M + 1:t-1}^{\neg (t-M)}$, the change in $\theta_{t-M}$ does not affect $H_t$ and therefore all the partial derivatives in Eq.~\ref{eqn:corr_term} become zero. This property saves us from expensive higher-order derivative computations or inducing permanent bias due to the negligence of the correction term. On the other hand, geometrically, our approach corresponds to choosing a flat Riemannian manifold specified by the metric tensor $H_t$, whose curvature is constant (i.e. $\Gamma(\theta) =0$).

 \begin{algorithm2e} [t!]
 \SetInd{0.000ex}{0.8ex}
 \DontPrintSemicolon
 \SetKwInOut{Input}{input}
 \Input{$M$, $\gamma$, $\lambda$, $N_\Omega$, $\theta_0,\dots,\theta_{2M}$}
 \For{$t=2M+1,\cdots, T$}{
    Draw $\Omega_t \subset \{ 1,\dots,N\}$ \\
    Generate $z_t \sim {\cal N}(0,I)$ \\
    {\color{mydarkblue} \tcp{L-BFGS computations}}
    $\xi_t = H_t\bigl({\theta}_{t-2M + 1:t-1}^{\neg (t-M)} \bigr) \nabla \tilde{U}_{\Omega_t}(\theta_{t-M})$ \hfill {\color{mydarkblue} (Eq.~\ref{eqn:lbfgs})}\\
    $\eta_t =  S_t \bigl({\theta}_{t-2M + 1:t-1}^{\neg (t-M)} \bigr) z_t$ \hfill  {\color{mydarkblue} (Eq.~\ref{eqn:lbfgs_fact})} \\
    {\color{mydarkblue} \tcp{Generate the new sample}}
    $\theta_t = \theta_{t-M} - \epsilon_t \xi_t + \sqrt{2 \epsilon_t} \eta_t$ \hfill  {\color{mydarkblue} (Eq.~\ref{eqn:hamcmc})} \\
    {\color{mydarkblue} \tcp{Update L-BFGS memory}}
    $s_t \hspace{-2pt} = \hspace{-2pt} \theta_t \hspace{-2pt} - \hspace{-2pt} \theta_{t-M}$, \hspace{1.5pt}  $y_t \hspace{-2pt} = \hspace{-2pt} \nabla \tilde{U}_{\Omega_t}(\theta_t) \hspace{-2pt}-\hspace{-2pt} \nabla \tilde{U}_{\Omega_t}(\theta_{t-M}) \hspace{-2pt} + \hspace{-2pt} \lambda s_t$ \\
    %Discard $s_{t-M+2}$ and $y_{t-M+2}$
 }
 \caption{Hessian Approximated MCMC}
 \label{algo:hamcmc}
 \end{algorithm2e}

We encounter several challenges due to the usage of stochastic gradients that are computed on random subsets of the data. The first challenge is the \emph{data consistency} \cite{schraudolph2007stochastic,byrd:2014}. Since we draw different data subsamples $\Omega_t$  at different epochs, the gradient differences $y_t$ would be simply `inconsistent'. Handling this problem in incremental QN settings (i.e., the data subsamples are selected via a deterministic mechanism, rather than being chosen at random) is not as challenging, however, the randomness in the stochastic setting prevents simple remedies such as computing the differences in a cyclic fashion. Therefore, in order to ensure consistent gradient differences, we slightly increase the computational needs of HAMCMC and perform an additional stochastic gradient computation at the end of each iteration, i.e. $y_t = \nabla \tilde{U}_{\Omega_t}(\theta_t) - \nabla \tilde{U}_{\Omega_t}(\theta_{t-M})$ where both gradients are evaluated on the same subsample $\Omega_t$.

In order to have a valid sampler, we are required to have positive definite $H_t(\cdot)$. However, even if $U(\theta)$ was strongly convex, due to data subsampling L-BFGS is no longer guaranteed to produce positive definite approximations. One way to address this problem is to use variance reduction techniques as presented in \cite{byrd:2014}; however, such an approach would introduce further computational burden. In this study, we follow a simpler approach and address this problem by approximating the inverse Hessian in a trust region that is parametrized by $\lambda$, i.e. we use L-BFGS to approximate $(\nabla^2 {U}(\theta) + \lambda I)^{-1}$. For large enough $\lambda$ this approach will ensure positive definiteness. The L-BFGS updates for this problem can be obtained by simply modifying the gradient differences as: $y_t \leftarrow  y_t + \lambda s_t$ \cite{schraudolph2007stochastic}. 

The final challenge is to generate $\eta_t$ whose covariance is a scaled version of the L-BFGS output. This requires to compute the square root of $H_t$, i.e. $S_t S_t^\top =  H_t$, so that we can generate $\eta_t$ as $\sqrt{2\epsilon_t}S_t z_t$, where $z_t \sim {\cal N}(0, I)$.  Fortunately, we can directly compute the product $S_t z_t$ within the L-BFGS framework by using the following recursion \cite{BRODLIE01021973,zhang2011quasi}: 
\begin{flalign}
\nonumber & \hspace{-1pt} H_t^m = S_t^m (S_t^m)^\top, \> \> S_t^m = (I - p_\tau q_\tau^\top) S_t^{m-1} \\
\nonumber & \hspace{-1pt} B_t^m = C_t^m (C_t^m)^\top, \> C_t^m = (I - u_\tau v_\tau^\top) C_t^{m-1} \\
\nonumber & v_\tau = \frac{s_\tau}{s_\tau^\top B_t^{m-1} s_\tau}, \> \hspace{5pt} u_\tau = \sqrt{\frac{s_\tau^\top B_t^{m-1} s_\tau}{s_\tau^\top  y_\tau} }  y_\tau + B_t^{m-1} s_\tau, \\
& p_\tau = \frac{s_\tau}{s_\tau^\top y_\tau}, \> \hspace{12pt} q_\tau = \sqrt{\frac{s_\tau^\top y_\tau}{s_\tau^\top B_t^{m-1} y_\tau} } B_t^{m-1} s_\tau - y_\tau \label{eqn:lbfgs_fact}
\end{flalign}
where $S_t \equiv S_t^{M-1}$ and $B_t^m = (H_t^m)^{-1}$. By this approach, $\eta_t$ can be computed in ${\cal O}(M^2 D)$ time, which is still linear in $D$. We illustrate HAMCMC in Algorithm~\ref{algo:hamcmc}. 

Note that large $M$ would result in a large gap between two consecutive samples and therefore might require larger $\lambda$ for ensuring positive definite L-BFGS approximations. Such a circumstance would be undesired since $H_t$ would get closer to the identity matrix and therefore result in HAMCMC being similar to SGLD. In our computational studies, we have observed that choosing $M$ between $2$ and $5$ is often adequate in several applications.

\subsection{Convergence Analysis}

We are interested in approximating the posterior expectations by using sample averages: 
\begin{align}
\bar{h} \triangleq \int_{\mathds{R}^D} h(\theta) \pi(d\theta) \approx \hat{h} \triangleq \frac1{W_T} \sum_{t=1}^T \epsilon_t h(\theta_t) ,
\end{align} 
where $\pi(\theta) \triangleq p(\theta|x)$, $W_T \triangleq \sum_{t=1}^T \epsilon_t$, and $\theta_t$ denotes the samples generated by HAMCMC. In this section we will analyze the asymptotic properties of this estimator. 
We consider a theoretical framework similar to the one presented in \cite{chen2015convergence}. In particular, we are interested in analyzing the bias $\bigl| \mathds{E}[\hat{h}] - \bar{h} \bigr|$ and the mean squared-error $\mathds{E}\bigl[ (\hat{h}-\bar{h})^2 \bigr]$ of our estimator. 

A useful way of analyzing dynamics-based MCMC approaches is to first analyze the underlying continuous dynamics, then consider the MCMC schema as a discrete-time approximation \cite{Roberts03,icml2014c2_satoa14,chen2015convergence}. However, this approach is not directly applicable in our case.
Inspired by \cite{Roberts1994287} and \cite{zhang2011quasi}, we convert Eq.~\ref{eqn:hamcmc} to a first order Markov chain that is defined on the product space $\mathds{R}^{D(2M-1)}$ such that $\Theta_t \equiv \{\theta_{t-2M+2},\dots,\theta_{t}\}$. In this way, we can show that HAMCMC uses a transition kernel which modifies only one element of $\Theta_t$ per iteration and rearranges the samples that will be used in the L-BFGS computations. Therefore, the overall HAMCMC kernel can be interpreted as a composition of $M$ different preconditioned SGLD kernels whose volatilities are independent of their current states. 
By considering this property and assuming certain conditions that are provided in the Supplement hold, we can bound the bias and MSE as shown in Theorem~\ref{thm:cnvg}.

\begin{thm}
\label{thm:cnvg}
Assume the number of iterations is chosen as $T = K M$ where $K \in \mathds{N}_+$ and $\{\theta_t\}_{t=1}^T$ are obtained by HAMCMC (Eq.~\ref{eqn:hamcmc}). Define $L_K \triangleq \sum_{k=1}^K \epsilon_{kM}$, $Y_K \triangleq \sum_{k=1}^K \epsilon_{kM}^2$, and the operator $\Delta V_t \triangleq (\nabla \tilde{U}(\theta_t) - \nabla U(\theta_t) )^\top H_t \nabla$. Then, we can bound the bias and MSE of HAMCMC as follows:
\begin{enumerate}[label=(\alph*),topsep=0pt,leftmargin=*]
% \item $d(\pi_t, \pi) = {\cal O}\bigl(\epsilon_t \bigr)$ %\hfill (Limiting distribution)
\item $\bigl| \mathds{E}[\hat{h}] - \bar{h} \bigr| = {\cal O}\bigl(\frac1{L_K} + \frac{Y_K}{L_K}\bigr)$ %\hfill (Bias)
\item $\mathds{E}\bigl[ (\hat{h}-\bar{h})^2 \bigr] = {\cal O}\Bigl(\sum\limits_{k=1}^K \frac{\epsilon_{kM}^2}{L_K^2} \mathds{E}\|\Delta V_{k^\star} \|^2 + \frac1{L_K } + \frac{Y_K^2 }{L_K^2} \Bigr) $ %\hfill (MSE)
\end{enumerate}
where $h$ is smooth and $ \Delta V_{k^\star} = \Delta V_{m_k^* + kM}$, where $m_k^* = \argmax_{1 \leq m \leq M}  \mathds{E}\|\Delta V_{m+kM} \|^2$, and $\|\Delta V_t\|$ denotes the operator norm.
\end{thm}
The proof is given in the Supplement. The theorem implies that as $T$ goes to infinity, the bias and the MSE of the estimator vanish. Therefore, HAMCMC is asymptotically unbiased and consistent.

\begin{figure}[t]
\centering
\hspace{-5pt} \includegraphics[width=0.49\columnwidth]{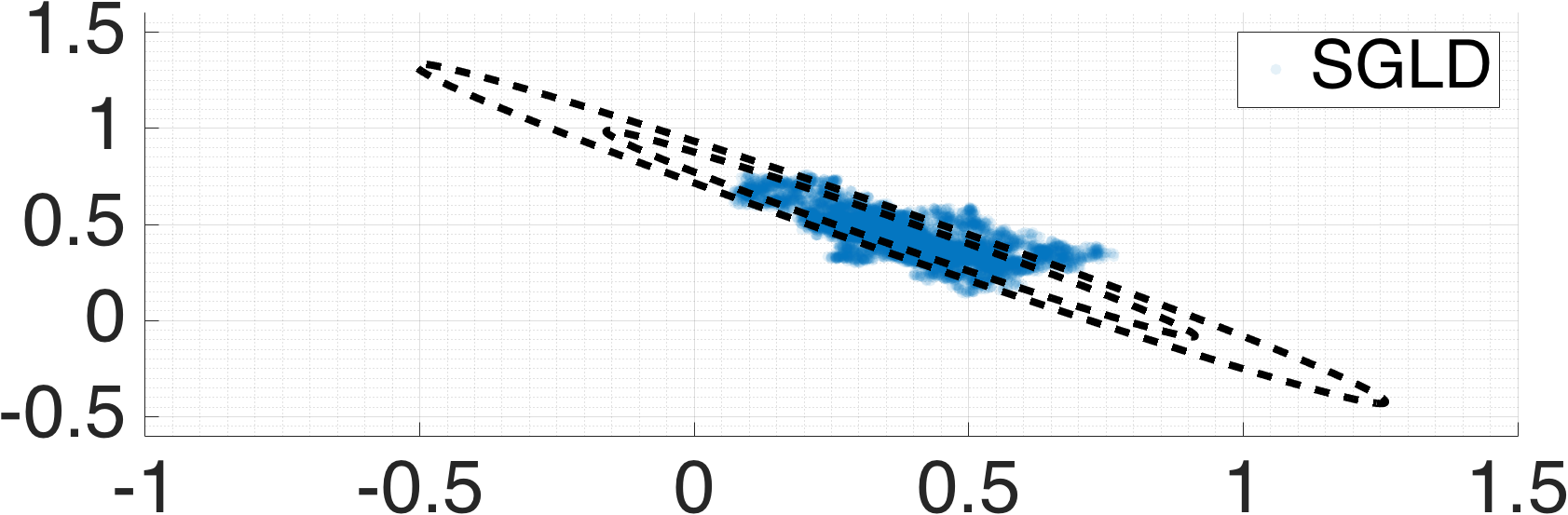}  
\includegraphics[width=0.49\columnwidth]{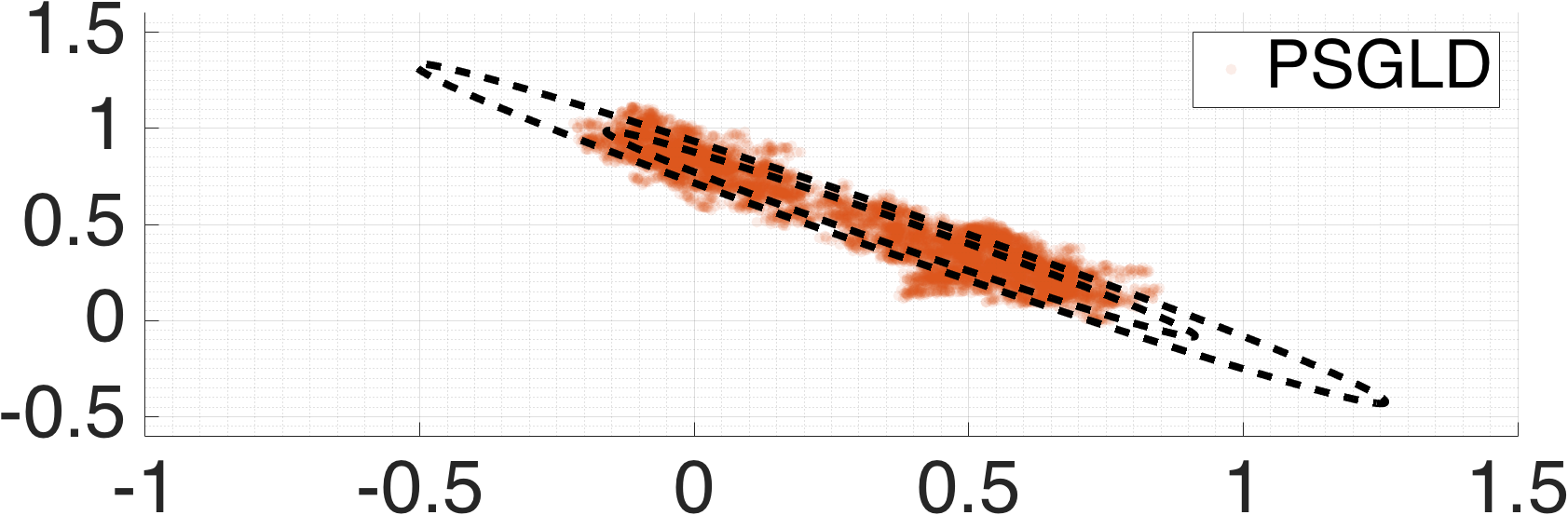}\\ %\hspace{0.1cm}
\includegraphics[width=0.49\columnwidth]{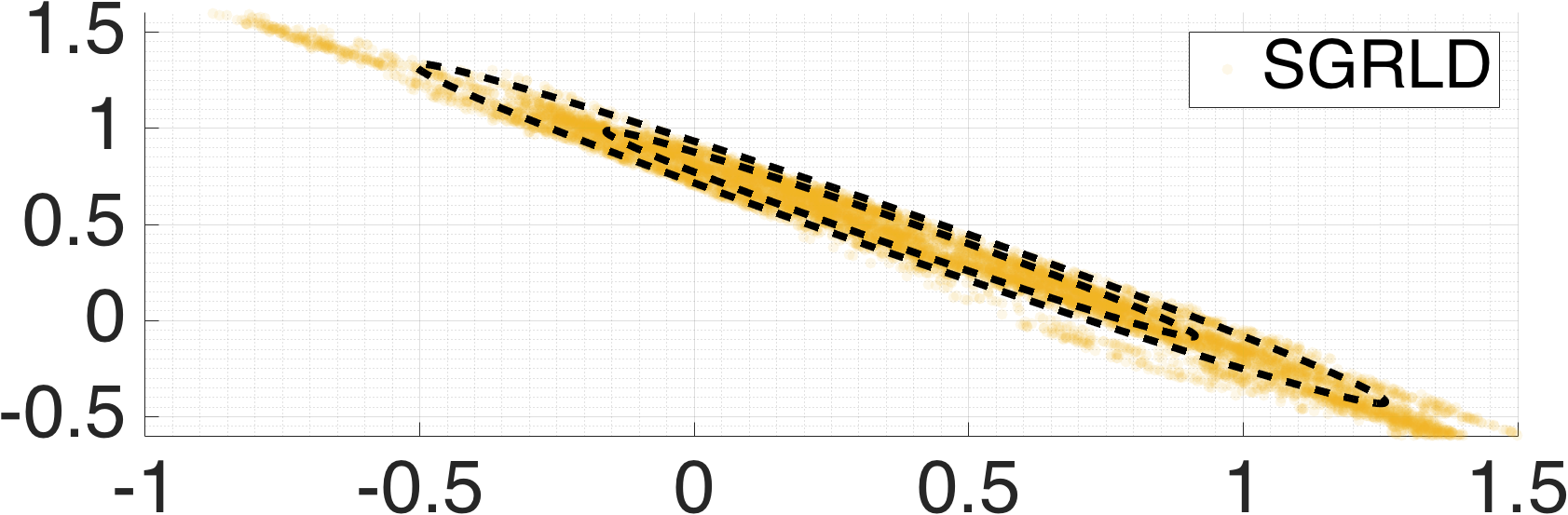} \hfill %\hspace{0pt}
\includegraphics[width=0.49\columnwidth]{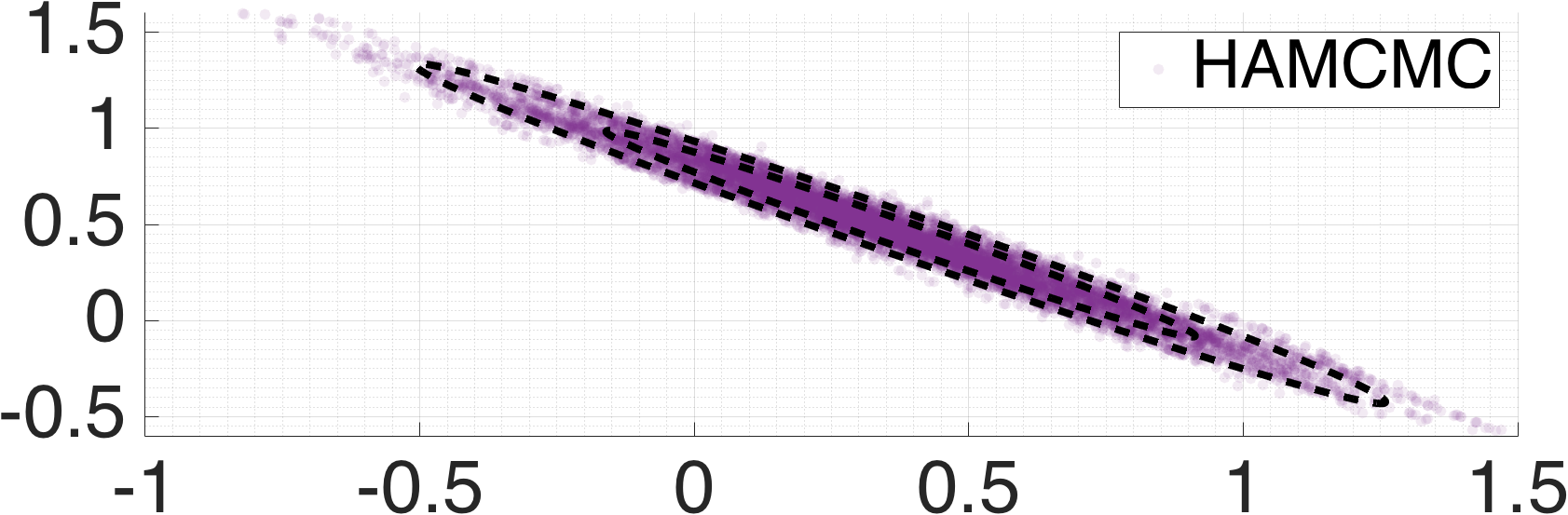}
\caption{The illustration of the samples generated by different SG-MCMC methods. The contours of the true posterior distribution is shown with dashed lines. We set $M=2$ for HAMCMC.}
\label{fig:synth2d}
\vspace{-5pt}
\end{figure}

\section{Experiments}
\label{sec:exp}

\subsection{Linear Gaussian Model}

We conduct our first set of experiments on synthetic data where we consider a rather simple Gaussian model whose posterior distribution is analytically available. The model is given as follows:
\begin{align}
\theta \sim {\cal N}(0, I), \quad x_n|\theta \sim {\cal N}(a_n^\top \theta, \sigma_x^2), \label{eqn:simplemodel}
\end{align}
for all $n$. Here, we assume $\{a_n\}_{n=1}^N$ and $\sigma_x^2$ are known and we aim to draw samples from the posterior distribution $p(\theta|x)$. We will compare the performance of HAMCMC against SGLD, PSGLD, and SGRLD. In these experiments, we determine the step size as $\epsilon_t = (a_\epsilon/t)^{0.51}$. In all our experiments, for each method, we have tried several values for the hyper-parameters with the same rigor and we report the best results. We also report all of the hyper-parameters used in the experiments (including Sections~\ref{sec:amf}-\ref{sec:dmf}) in the Supplement. These experiments are conducted on a standard laptop computer with $2.5$GHz Quad-core Intel Core i$7$ CPU and all the methods are implemented in Matlab.

In the first experiment, we first generate the true $\theta$ and the observations $x$ by using the generative model given in Eq.~\ref{eqn:simplemodel}. We set $D=2$ for visualization purposes and $\sigma_x^2 = 10$, then we randomly generate the vectors $\{a_n\}_n$ in such a way that the posterior distribution will be correlated. Then we generate $T = 20000$ samples by using the SG-MCMC methods, where the size of the data subsamples is selected as $N_\Omega = N/100$. We discard the first half of the samples as burn-in. Figure~\ref{fig:synth2d} shows the typical results of this experiment. We can clearly observe that SGLD cannot explore the posterior efficiently and gets stuck around the mode of the distribution. PSGLD captures the shape of the distribution in a more accurate way than SGLD since it is able to take the scale differences into account; however, it still cannot explore lower probability areas efficiently. Besides, we can see that SGRLD performs much better than SGLD and PSGLD. This is due to the fact that it uses the inverse expected FIM of the \emph{full} problem for generating each sample, where this matrix is simply $(\sum_n a_n a_n^\top/\sigma_x^2 + I)^{-1}$ for this model. Therefore, SGRLD requires much heavier computations as it needs to store a $D\times D$ matrix, computes the Cholesky factors of this matrix, and performs matrix-vector products at each iteration. Furthermore, in more realistic probabilistic models, the expected FIM almost always depends on $\theta$; hence, the computation burden of SGRLD is further increased by the computation of the inverse of the $D\times D$ matrix at each iteration. On the other hand, Fig.~\ref{fig:synth2d} shows that HAMCMC takes the best of both worlds: it is able to achieve the quality of SGRLD since it makes use of dense approximations to the inverse Hessian and it has low computational requirements similar to SGLD and PSGLD. This observation becomes more clear in our next experiment.

\begin{figure}[t]
\centering
\includegraphics[width=0.49\columnwidth]{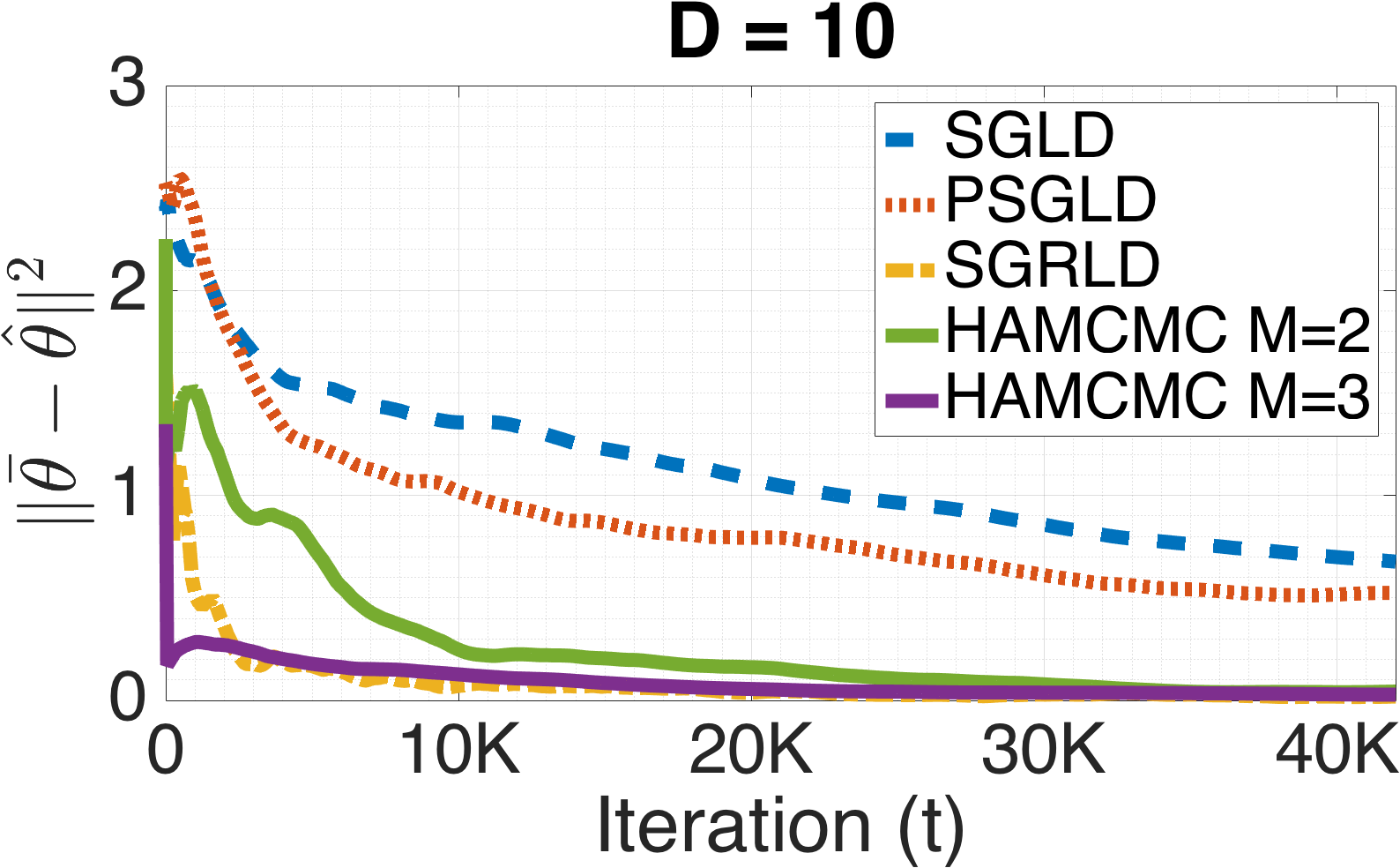}  %\hspace{0.1cm}
\includegraphics[width=0.49\columnwidth]{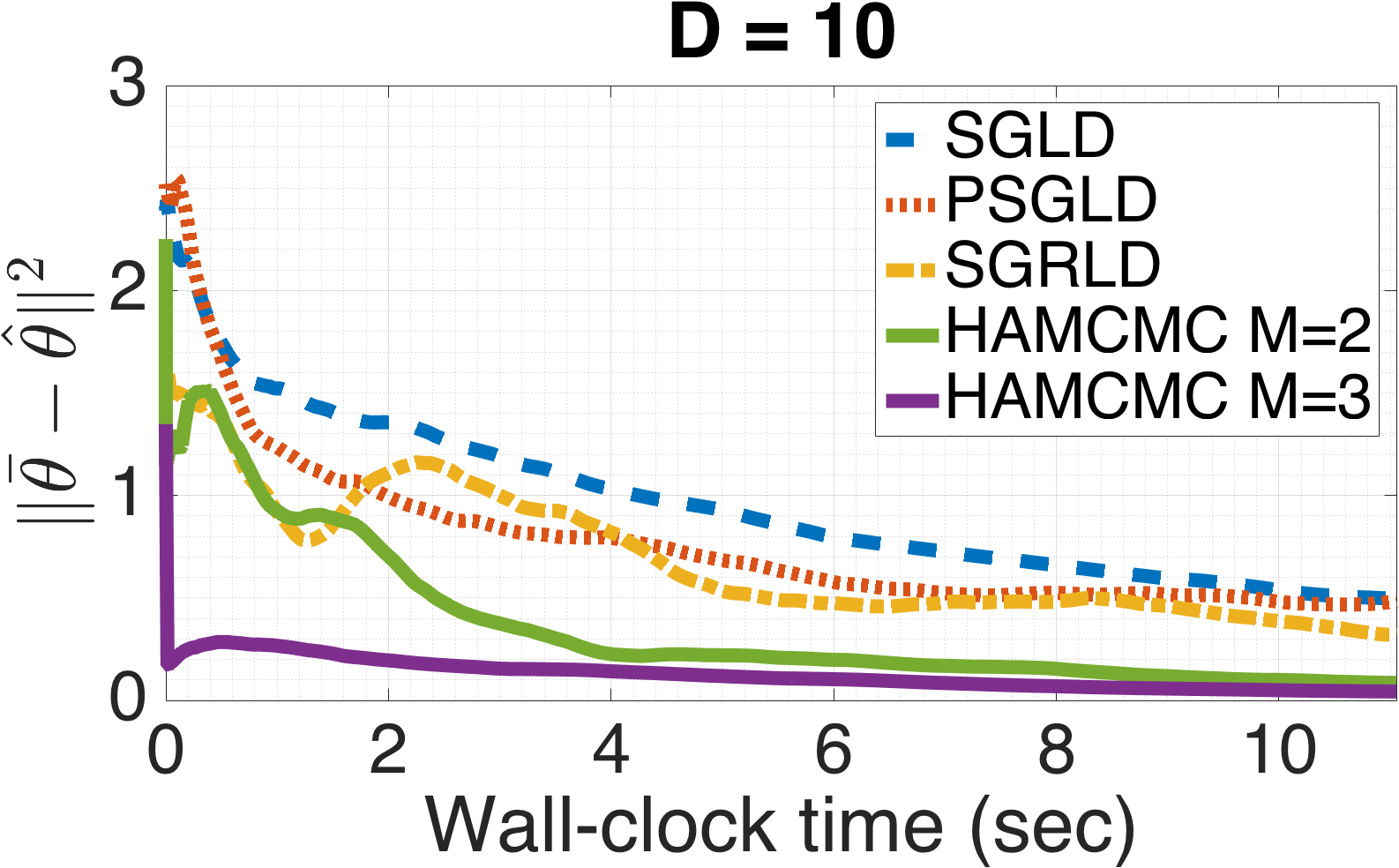} \vspace{0.1cm} \\ 
\includegraphics[width=0.49\columnwidth]{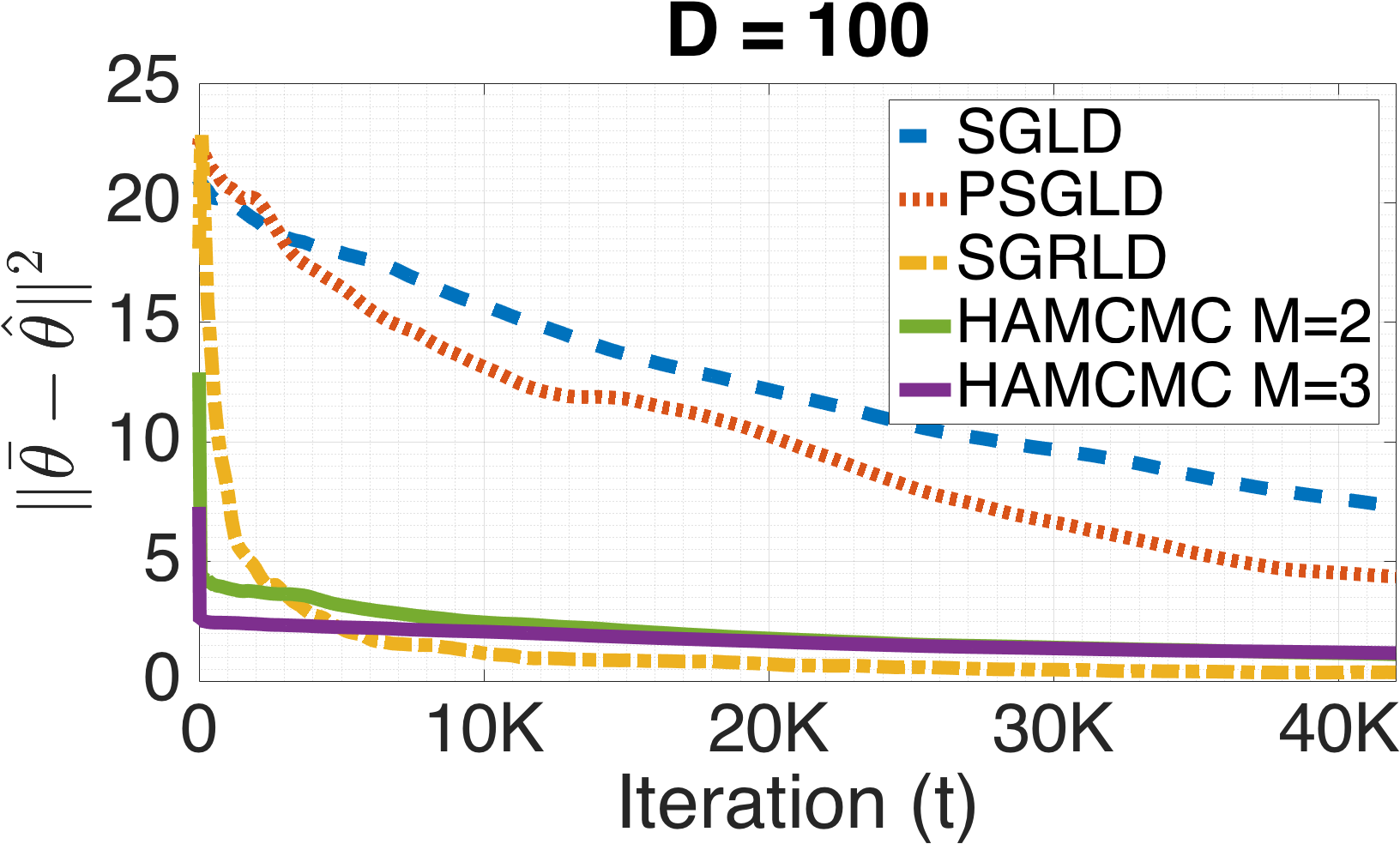}  %\hspace{0.1cm}
  %\hspace{0.1cm}
\includegraphics[width=0.49\columnwidth]{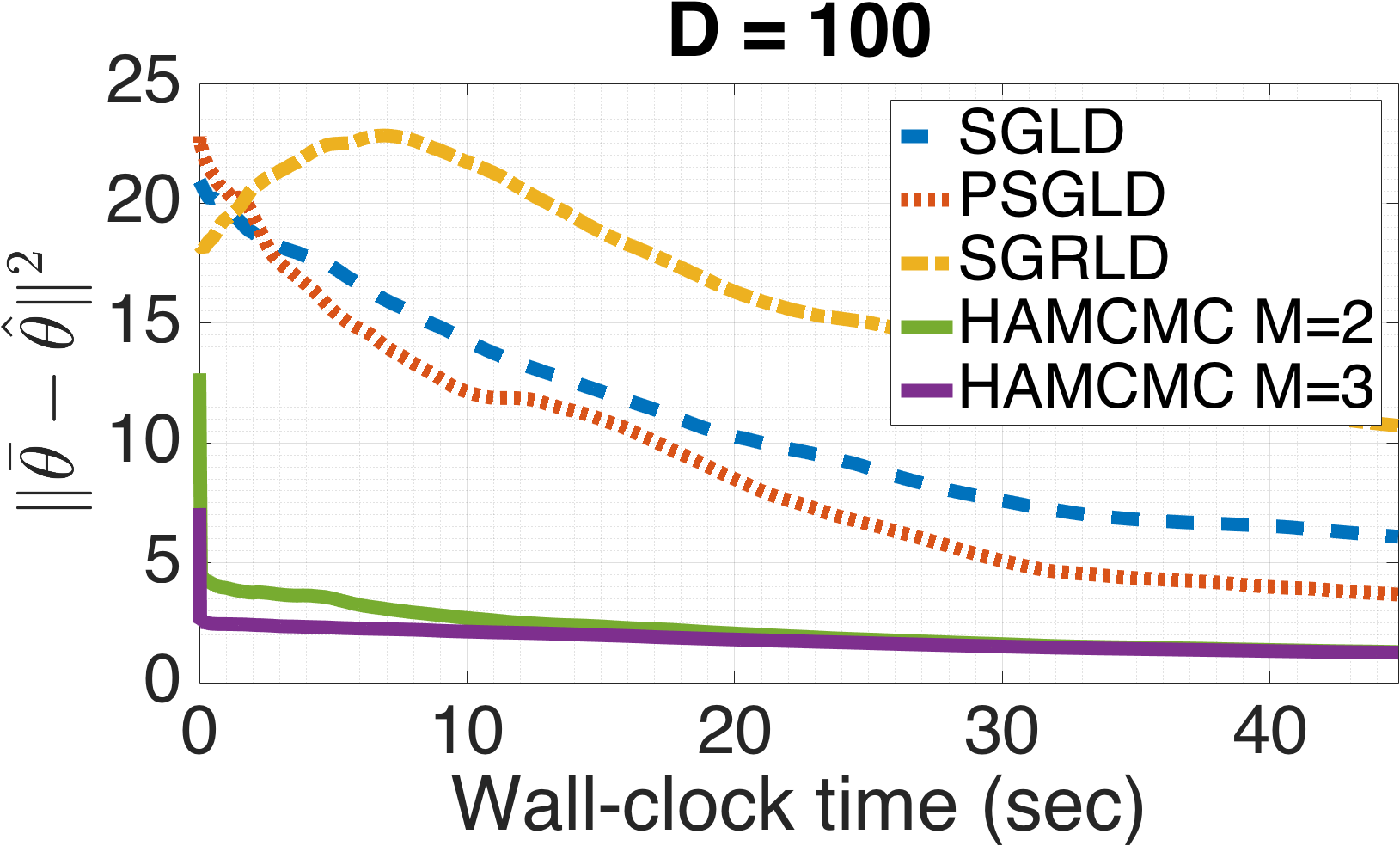} 
\caption{The error between the true posterior mean $\bar{\theta}$ and the Monte Carlo estimates $\hat{\theta}$ for $D=10$ and $D=100$. On the left column, we show the error versus iterations and on the right column we show the error versus wall-clock time.}
\label{fig:synth_err}
\vspace{-10pt}
\end{figure}

In our second experiment, we again generate the true $\theta$ and the observations by using Eq.~\ref{eqn:simplemodel}. Then, we approximate the posterior mean $\bar{\theta} = \int \theta p(\theta|x) d\theta$ by using the aforementioned MCMC methods. For comparison, we monitor $\|\bar{\theta} - \hat{\theta}\|^2$ where $\hat{\theta}$ is the estimate that is obtained from MCMC. Figure~\ref{fig:synth_err} shows the results of this experiment for $D=10$ and $D=100$. In the left column, we show the error as a function of iterations. These results show that the convergence rate of SGRLD (with respect to iterations) is much faster than SGLD and PSGLD. As expected, HAMCMC yields better results than SGLD and PSGLD even in the simplest configuration ($M=2$), and it performs very similarly to SGRLD as we set $M=3$. The performance difference becomes more prominent as we increase $D$.  

An important advantage of HAMCMC is revealed when we take into account the total running time of these methods. These results are depicted in the right column of Fig.~\ref{fig:synth_err}, where we show the error as a function of wall-clock time \footnote{The inverse of the FIM for this model can be precomputed. However, this will not be the case in more general scenarios. Therefore, for fair comparison we perform the matrix inversion and Cholesky decomposition operations at each iteration.}. We can observe that, since SGLD, PSGLD, and HAMCMC have similar computational complexities, the shapes of their corresponding plots do not differ significantly. However, SGRLD appears to be much slower than the other methods as we increase $D$.

\subsection{Alpha-Stable Matrix Factorization}
\label{sec:amf}

In our second set of experiments, we consider a recently proposed probabilistic matrix factorization model, namely alpha-stable matrix factorization ($\alpha$MF), that is given as follows \cite{simsekli2015alpha}: 
\begin{gather}
\nonumber W_{ik} \sim {\cal GG}(a_w,b_w,-2/\alpha), \quad H_{kj} \sim {\cal GG}(a_h,b_h,-2/\alpha)\\
X_{ij} | W,H \sim {\cal S}\alpha{\cal S}_c\bigl([\sum\nolimits_k W_{ik} H_{kj}]^{1/\alpha} \bigr) \label{eqn:amf}
\end{gather}
where $X \in \mathds{C}^{I\times J}$ is the observed matrix with \emph{complex} entries, and $W \in \mathds{R}_+^{I \times K}$ and $H \in \mathds{R}_+^{K \times J}$ are the latent factors. Here ${\cal GG}$ denotes the generalized gamma distribution \cite{stacy1962} and ${\cal S}\alpha{\cal S}_c$ denotes the complex symmetric $\alpha$-stable distribution \cite{samoradnitsky1994stable}. 
Stable distributions are heavy-tailed distributions and they are the limiting distributions in the generalized central limit theorem. They are often characterized by four parameters: ${\cal S}(\alpha, \beta, \sigma,\mu)$, where 
$\alpha \in (0,2]$ is the characteristic exponent, $\beta \in [-1 ,1]$ is the skewness parameter, $\sigma \in (0,\infty)$ is the dispersion parameter, and $\mu \in (-\infty, \infty)$ is the location parameter. The distribution is called symmetric (${\cal S}\alpha{\cal S}$) if $\beta = 0$. The location parameter is also chosen as $0$ in $\alpha$MF.

The probability density function of the stable distributions cannot be written in closed-form except for certain special cases; however it is easy to draw samples from them. Therefore, MCMC has become the de facto inference tool for $\alpha$-stable models like $\alpha$MF \cite{godsill1999bayesian}. By using data augmentation and the product property of stable distributions, \citet{simsekli2015alpha} proposed a partially collapsed Gibbs sampler for making inference in this model. The performance of this approach was then illustrated on a speech enhancement problem.

In this section, we derive a new inference algorithm for $\alpha$MF that is based on SG-MCMC. By using the product property of stable models \cite{kuruoglu1999signal}, we express the model as conditionally Gaussian, given as follows:
\begin{gather}
\nonumber W_{ik} \sim {\cal GG}(a_w,b_w,-2/\alpha), \quad H_{kj} \sim {\cal GG}(a_h,b_h,-2/\alpha)\\
\nonumber \hspace{26pt} \Phi_{ij} \sim {\cal S} \bigl( {\alpha}/{2},1, 2 \bigl(\cos (\pi \alpha)/{4} \bigr)^{2/\alpha},0 \bigr)\\
 X_{ij} | W,H,\Phi \sim {\cal N}_c\bigl(0, \Phi_{ij}[\sum\nolimits_k W_{ik} H_{kj}]^{2/\alpha} \bigr),
\end{gather}
where ${\cal N}_c$ denotes the complex Gaussian distribution. Here, by assuming $\alpha$ is already provided, we propose a Metropolis and SG-MCMC within Gibbs approach, where we generate samples from the full conditional distributions of the latent variables in such a way that we use a Metropolis step for sampling $\Phi$ where the proposal is the prior and we use SG-MCMC for sampling $W$ and $H$ from their full conditionals. Since all the elements of $W$ and $H$ must be non-negative, we apply mirroring at the end of each iteration where we replace negative elements with their absolute values \cite{PatTeh2013a}.

\begin{figure}[t]  
\centering
\includegraphics[width=0.8\columnwidth]{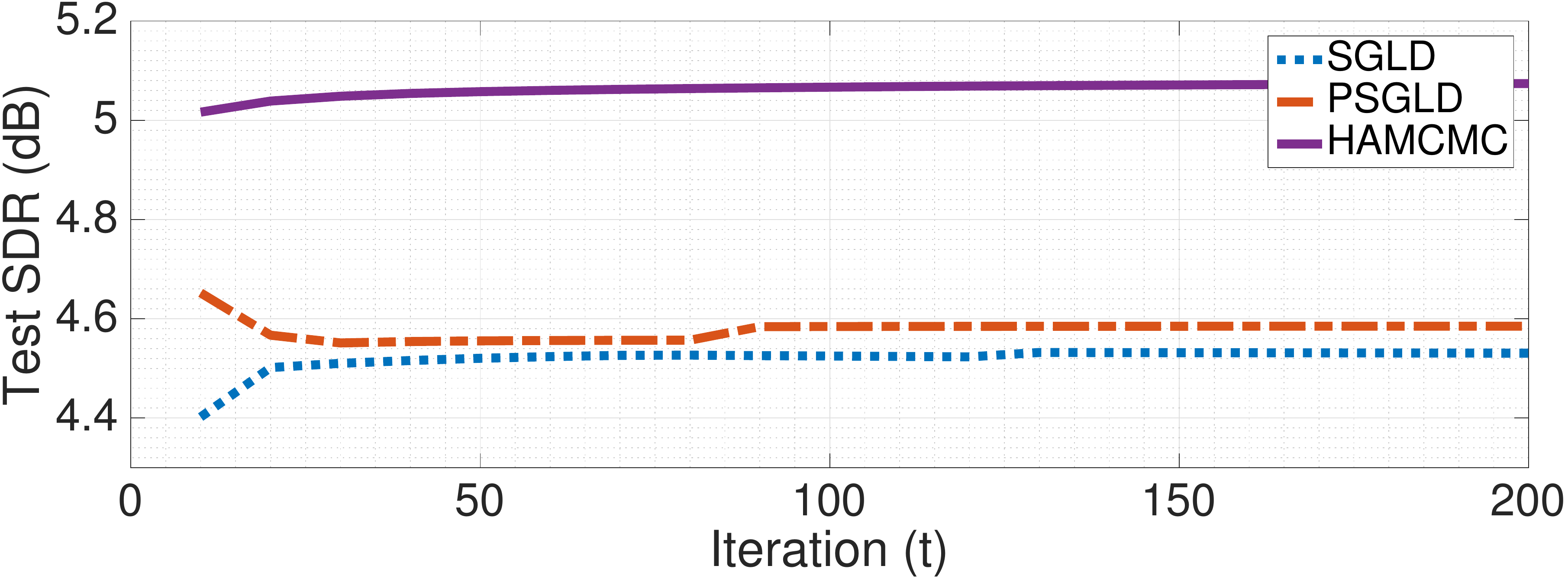}  %\hspace{0.1cm}
\caption{The performance of HAMCMC on a speech denoising application in terms of SDR (higher is better). }
\label{fig:audio}
\vspace{-9pt}
\end{figure}

We evaluate HAMCMC on the same speech enhancement application described in \cite{simsekli2015alpha} by using the same settings. The aim here is to recover clean speech signals that are corrupted by different types of real-life noise signals \cite{hu2007subjective}. We use the Matlab code that can be downloaded from the authors' websites in order to generate the same experimental setting. We first train $W$ on a clean speech corpus, then we fix $W$ at denoising and sample the rest of the variables. For each test signal (out of $80$) we have $I=256$ and $J=140$. The rank is set $K=105$ where the first $100$ columns of $W$ is reserved for clean speech and for speech $\alpha = 1.2$ and for noise $\alpha =1.89$. Note that the model used for denoising is slightly different from Eq.~\ref{eqn:amf} and we refer the readers to the original paper for further details of the experimental setup.

In this experiment, we compare HAMCMC against SGLD and PSGLD. SGRLD cannot be applied to this model since the expected FIM is not tractable. For each method, we set $N_\Omega = IJ/10$ and we generate $2500$ samples where we compute the Monte Carlo averages by using the last $200$ samples. For HAMCMC, we set $M=5$. Fig.~\ref{fig:audio} shows the denoising results on the test set in terms of signal-to-distortion ratio (SDR) when the mixture signal-to-noise ratio is $5$dB. The results show that SGLD and PSGLD perform similarly, where the quality of the outputs of PSGLD is slightly higher than the ones obtained from SGLD. We can observe that HAMCMC provides a significant performance improvement over SGLD and PSGLD, thanks to the usage of local curvature information.

\subsection{Distributed Matrix Factorization}
\label{sec:dmf}

In our final set of experiments, we evaluate HAMCMC on a distributed matrix factorization problem, where we consider the following probabilistic model:
% \begin{gather}
% \nonumber W_{ik} \sim {\cal N}(0,\sigma_w^2), \quad H_{kj} \sim {\cal N}(0,\sigma_h^2)\\
% X_{ij} | W,H \sim {\cal N}\bigl(\sum\nolimits_k W_{ik} H_{kj}, \sigma^2_x \bigr).
% \end{gather}
% \begin{gather}
% \begin{align*}
$W_{ik} \sim {\cal N}(0,\sigma_w^2)$, $H_{kj} \sim {\cal N}(0,\sigma_h^2)$, $ X_{ij} | \cdot \sim {\cal N}\bigl(\sum\nolimits_k W_{ik} H_{kj}, \sigma^2_x \bigr) $.
% \end{align*}
% \end{gather}
Here, $W \in \mathds{R}^{I \times K}$, $H \in \mathds{R}^{K \times J}$, and $X \in \mathds{R}^{I \times J}$. This model is similar to \cite{salakhutdinov2008bayesian} and is often used in distributed matrix factorization problems \cite{gemulla2011,csimcsekli2015hamsi}.

\begin{figure}[t]
\centering
\includegraphics[width=0.7\columnwidth]{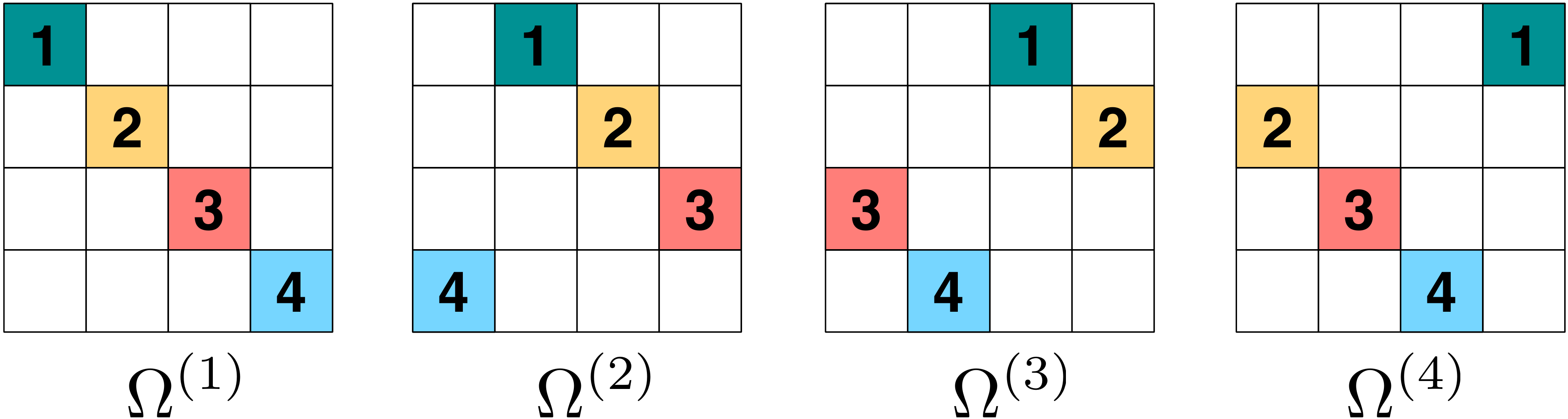}  %\hspace{0.1cm}
\caption{The illustration of the data partitioning that is used in the distributed matrix factorization experiments. The numbers in the blocks denote the computer that stores the corresponding block. }
\label{fig:illus}
\vspace{-14pt}
\end{figure}

Inspired by Distributed SGLD (DSGLD) for matrix factorization \cite{Ahn15} and Parallel SGLD \cite{csimcsekli2015parallel}, we partition $X$ into disjoint subsets where each subset is further partitioned into disjoint blocks. This schema is illustrated in Fig.~\ref{fig:illus}, where $X$ is partitioned into $4$ disjoint subsets $\Omega^{(1)},\dots,\Omega^{(4)}$ and each subset contains $4$ different blocks that will be distributed among different computation nodes. By using this approach, we extend PSGLD and HAMCMC to the distributed setting similarly to DSGLD, where at each iteration the data subsample $\Omega_t$ is selected as $\Omega^{(i)}$ with probability $|\Omega^{(i)}|/\sum_j |\Omega^{(j)}|$. At the end of each iteration the algorithms transfer a small block of $H$ to a neighboring node depending on $\Omega_{t+1}$. PSGLD and HAMCMC further need to communicate other variables that are required for computing $H_t$, where the communication complexity is still linear with $D$. Finally, HAMCMC requires to compute six vector dot products at each iteration that involves all the computation nodes, such as $s_t^\top y_t$. These computations can be easily implemented by using a \emph{reduce} operation which has a communication complexity of ${\cal O}(M^2)$. By using this approach, the methods have the same theoretical properties except that the stochastic gradients would be expected to have larger variances.

For this experiment, we have implemented all the algorithms in C by using a low-level message passing protocol, namely the OpenMPI library. In order to keep the implementation simple, for HAMCMC we set $M=2$. We conduct our experiments on a small-sized cluster with $4$ interconnected computers each of them with $4$ Intel Xeon $2.93$GHz CPUs and $192$ GB of memory.

\begin{figure}
\centering
\includegraphics[width=0.8\columnwidth]{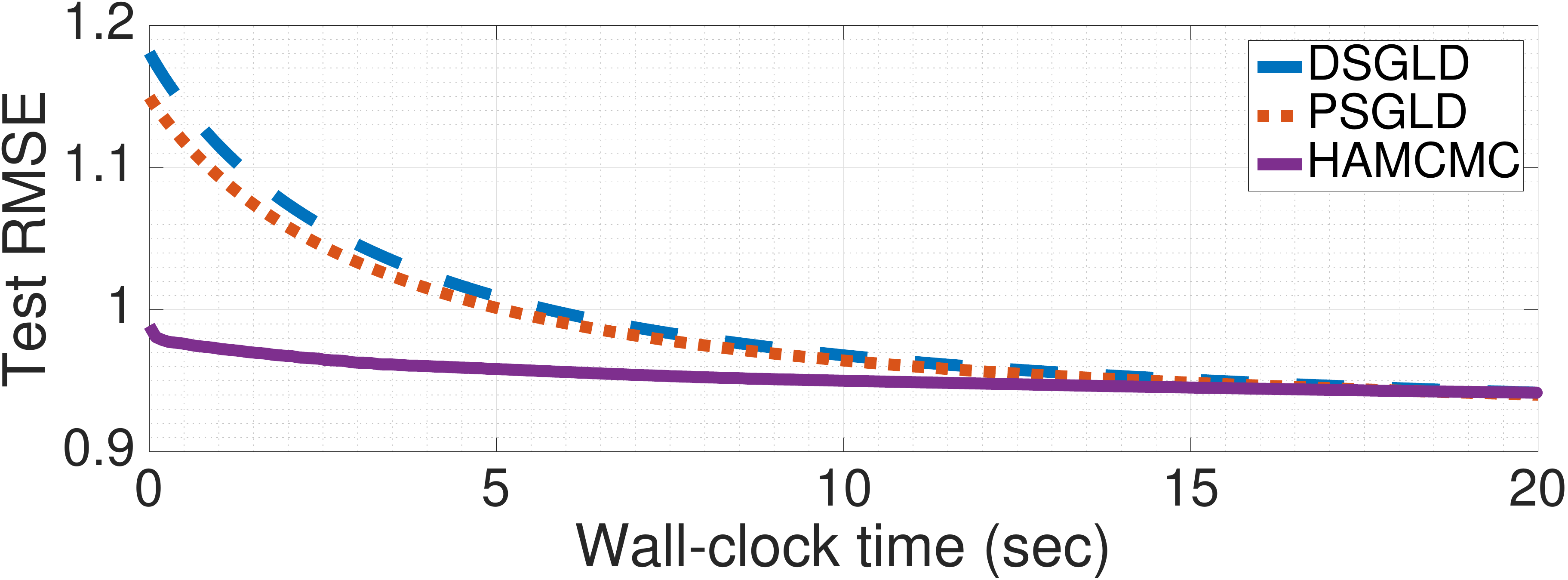}  %\hspace{0.1cm}
\caption{The performance of HAMCMC on a distributed matrix factorization problem.}
\vspace{-12.5pt}
\label{fig:rmse}
\end{figure}

We compare distributed HAMCMC against DSGLD and distributed PSGLD on a large movie ratings dataset, namely the MovieLens $1$M (\url{grouplens.org}). This dataset contains about $1$ million ratings applied to $I = 3883$ movies by $J = 6040$ users, resulting in a sparse observed matrix $X$ with $4.3\%$ non-zero entries. We randomly select $10\%$ of the data as the test set and partition the rest of the data as illustrated in Fig.~\ref{fig:illus}. The rank of the factorization is chosen as $K = 10$. We set $\sigma_w^2 = \sigma_h^2 = \sigma_x^2 = 1$. In these experiments, we use a constant step size for each method and discard the first $50$ samples as burn-in. Note that for constant step size, we no longer have asymptotic unbiasedness; however, the bias and MSE are still bounded.

Fig.~\ref{fig:rmse} shows the comparison of the algorithms in terms of the root mean squared-error (RMSE) that are obtained on the test set. We can observe that DSGLD and PSGLD have similar converges rates. On the other hand, HAMCMC converges faster than these methods while having almost the same computational requirements.

\vspace{-5pt}

\section{Conclusion}

We presented HAMCMC, a novel SG-MCMC algorithm that achieves high convergence rates by taking the local geometry into account via the local Hessian that is efficiently computed by the L-BFGS algorithm. HAMCMC is both efficient and powerful since it uses dense approximations of the inverse Hessian while having linear time and memory complexity. We provided formal theoretical analysis and showed that HAMCMC is asymptotically consistent. Our experiments showed that HAMCMC achieves better convergence rates than the state-of-the-art while keeping the computational cost almost the same. As a next step, we will explore the use of HAMCMC in model selection applications \cite{simsekli2016stochastic}.

\clearpage

\section*{Acknowledgments}
The authors would like to thank to \c S. \.{I}lker Birbil and Figen \"{O}ztoprak for fruitful discussions on Quasi-Newton methods. The authors would also like to thank to Charles Sutton for his explanations on the Ensemble Chain Adaptation framework. This work is partly supported by the French National Research Agency (ANR) as a part of the EDISON 3D project (ANR-13-CORD-0008-02). A. Taylan Cemgil is supported by T\"{U}B\.{I}TAK 113M492 (Pavera) and Bo\u{g}azi\c ci University BAP 10360-15A01P1. 

\bibliography{./sgmcmc}

\begin{thebibliography}{45}
\providecommand{\natexlab}[1]{#1}
\providecommand{\url}[1]{\texttt{#1}}
\expandafter\ifx\csname urlstyle\endcsname\relax
  \providecommand{\doi}[1]{doi: #1}\else
  \providecommand{\doi}{doi: \begingroup \urlstyle{rm}\Url}\fi

\bibitem[Ahn et~al.(2012)Ahn, Korattikara, and Welling]{AhnKorWel2012}
Ahn, S., Korattikara, A., and Welling, M.
\newblock Bayesian posterior sampling via stochastic gradient {F}isher scoring.
\newblock In \emph{International Conference on Machine Learning}, 2012.

\bibitem[Ahn et~al.(2014)Ahn, Shahbaba, and Welling]{AhnShaWel2014}
Ahn, S., Shahbaba, B., and Welling, M.
\newblock Distributed stochastic gradient {MCMC}.
\newblock In \emph{International Conference on Machine Learning}, 2014.

\bibitem[Ahn et~al.(2015)Ahn, Korattikara, Liu, Rajan, and Welling]{Ahn15}
Ahn, S., Korattikara, A., Liu, N., Rajan, S., and Welling, M.
\newblock Large-scale distributed {B}ayesian matrix factorization using
  stochastic gradient {MCMC}.
\newblock In \emph{International Conference on Knowledge Discovery and Data
  Mining}, August 2015.

\bibitem[Brodlie et~al.(1973)Brodlie, Gourlay, and Greenstadt]{BRODLIE01021973}
Brodlie, K.~W., Gourlay, A.~R., and Greenstadt, J.
\newblock Rank-one and rank-two corrections to positive definite matrices
  expressed in product form.
\newblock \emph{IMA J APPL MATH}, 11\penalty0 (1):\penalty0 73--82, 1973.

\bibitem[Bubeck et~al.(2015)Bubeck, Eldan, and Lehec]{bubeck2015finite}
Bubeck, S., Eldan, R., and Lehec, J.
\newblock Finite-time analysis of projected {Langevin Monte Carlo}.
\newblock In \emph{Advances in Neural Information Processing Systems}, pp.\
  1243--1251, 2015.

\bibitem[Bui-Thanh \& Ghattas(2012)Bui-Thanh and Ghattas]{buithanh2012scaled}
Bui-Thanh, T. and Ghattas, O.
\newblock A scaled stochastic {N}ewton algorithm for {M}arkov {C}hain {M}onte
  {C}arlo simulations.
\newblock \emph{SIAM Journal on Uncertainty Quantification}, 2012.

\bibitem[Byrd et~al.(1994)Byrd, Nocedal, and Schnabel]{byrd1994representations}
Byrd, R.~H., Nocedal, J., and Schnabel, R.~B.
\newblock Representations of quasi-{N}ewton matrices and their use in limited
  memory methods.
\newblock \emph{Mathematical Programming}, 63\penalty0 (1-3):\penalty0
  129--156, 1994.

\bibitem[Byrd et~al.(2014)Byrd, Hansen, Nocedal, and Singer]{byrd:2014}
Byrd, R.~H., Hansen, S.~L., Nocedal, J., and Singer, Y.
\newblock A stochastic quasi-{N}ewton method for large-scale optimization.
\newblock \emph{arXiv preprint arXiv:1401.7020}, 2014.

\bibitem[Calderhead \& Sustik(2012)Calderhead and Sustik]{calderhead2012sparse}
Calderhead, B. and Sustik, M.
\newblock Sparse approximate manifolds for differential geometric {MCMC}.
\newblock In \emph{Advances in Neural Information Processing Systems}, pp.\
  2879--2887, 2012.

\bibitem[Chaari et~al.(2015)Chaari, Tourneret, Chaux, and
  Batatia]{nonsmoothhmc}
Chaari, L., Tourneret, J.Y., Chaux, C., and Batatia, H.
\newblock A {H}amiltonian {M}onte {C}arlo method for non-smooth energy
  sampling.
\newblock \emph{arXiv preprint arXiv:1401.3988}, 2015.

\bibitem[Chen et~al.(2015)Chen, Ding, and Carin]{chen2015convergence}
Chen, C., Ding, N., and Carin, L.
\newblock On the convergence of stochastic gradient {MCMC} algorithms with
  high-order integrators.
\newblock In \emph{Advances in Neural Information Processing Systems}, pp.\
  2269--2277, 2015.

\bibitem[Chen et~al.(2014)Chen, Fox, and Guestrin]{ChenICML2014}
Chen, T., Fox, E.~B., and Guestrin, C.
\newblock Stochastic gradient {H}amiltonian {M}onte {C}arlo.
\newblock In \emph{International Conference on Machine Learning}, 2014.

\bibitem[\c{S}im\c{s}ekli et~al.(2015)\c{S}im\c{s}ekli, Liutkus, and
  Cemgil]{simsekli2015alpha}
\c{S}im\c{s}ekli, U., Liutkus, A., and Cemgil, A.~T.
\newblock Alpha-stable matrix factorization.
\newblock \emph{IEEE Signal Processing Letters}, 22\penalty0 (12):\penalty0
  2289--2293, 2015.

\bibitem[\c{S}im\c{s}ekli et~al.(2016)\c{S}im\c{s}ekli, Badeau, Richard, and
  Cemgil]{simsekli2016stochastic}
\c{S}im\c{s}ekli, U., Badeau, R., Richard, G., and Cemgil, A.~T.
\newblock Stochastic thermodynamic integration: efficient {B}ayesian model
  selection via stochastic gradient {MCMC}.
\newblock In \emph{41st International Conference on Acoustics, Speech and
  Signal Processing}, 2016.

\bibitem[Dahlin et~al.(2015)Dahlin, Lindsten, and Sch{\"o}n]{dahlin2015quasi}
Dahlin, J., Lindsten, F., and Sch{\"o}n, T.~B.
\newblock Quasi-{N}ewton particle {M}etropolis-{H}astings.
\newblock In \emph{IFAC Symposium on System Identification}, 2015.

\bibitem[Ding et~al.(2014)Ding, Fang, Babbush, Chen, Skeel, and
  Neven]{DingFBCSN14}
Ding, N., Fang, Y., Babbush, R., Chen, C., Skeel, R.~D., and Neven, H.
\newblock Bayesian sampling using stochastic gradient thermostats.
\newblock In \emph{Advances in Neural Information Processing Systems}, pp.\
  3203--3211, 2014.

\bibitem[Gemulla et~al.(2011)Gemulla, Nijkamp, J., and Sismanis]{gemulla2011}
Gemulla, R., Nijkamp, E., J., Haas.~P., and Sismanis, Y.
\newblock Large-scale matrix factorization with distributed stochastic gradient
  descent.
\newblock In \emph{ACM SIGKDD}, 2011.

\bibitem[Girolami \& Calderhead(2011)Girolami and
  Calderhead]{girolami2011riemann}
Girolami, M. and Calderhead, B.
\newblock Riemann manifold {Langevin and Hamiltonian Monte Carlo} methods.
\newblock \emph{J R Stat Soc Series B Stat Methodol}, 73\penalty0 (2):\penalty0
  123--214, 2011.

\bibitem[Godsill \& Kuruoglu(1999)Godsill and Kuruoglu]{godsill1999bayesian}
Godsill, S. and Kuruoglu, E.~E.
\newblock Bayesian inference for time series with heavy-tailed symmetric
  $\alpha$-stable noise processes.
\newblock \emph{Heavy Tails’ 99, Applications of Heavy Tailed Distributions
  in Economics, Engineering and Statistics}, pp.\  3--5, 1999.

\bibitem[Hu \& Loizou(2007)Hu and Loizou]{hu2007subjective}
Hu, Y. and Loizou, P.~C.
\newblock Subjective comparison and evaluation of speech enhancement
  algorithms.
\newblock \emph{Speech communication}, 49\penalty0 (7):\penalty0 588--601,
  2007.

\bibitem[Kuruoglu(1999)]{kuruoglu1999signal}
Kuruoglu, E.~E.
\newblock \emph{Signal processing in $\alpha$-stable noise environments: a
  least lp-norm approach}.
\newblock PhD thesis, University of Cambridge, 1999.

\bibitem[Kushner \& Yin(2003)Kushner and Yin]{kushner}
Kushner, H. and Yin, G.
\newblock \emph{{Stochastic Approximation and Recursive Algorithms and
  Applications}}.
\newblock Springer, New York, 2003.

\bibitem[Li et~al.(2016)Li, Chen, Carlson, and Carin]{li2015preconditioned}
Li, C., Chen, C., Carlson, D., and Carin, L.
\newblock Preconditioned stochastic gradient {Langevin} dynamics for deep
  neural networks.
\newblock In \emph{AAAI Conference on Artificial Intelligence}, 2016.

\bibitem[Ma et~al.(2015)Ma, Chen, and Fox]{ma2015complete}
Ma, Y.~A., Chen, T., and Fox, E.
\newblock A complete recipe for stochastic gradient {MCMC}.
\newblock In \emph{Advances in Neural Information Processing Systems}, pp.\
  2899--2907, 2015.

\bibitem[Neal(2010)]{neal2010}
Neal, R.~M.
\newblock {MCMC} using {Hamiltonian} dynamics.
\newblock \emph{Handbook of {Markov Chain Monte Carlo}}, 54, 2010.

\bibitem[Nocedal \& Wright(2006)Nocedal and Wright]{nocedal}
Nocedal, J. and Wright, S.~J.
\newblock \emph{Numerical optimization}.
\newblock Springer, Berlin, 2006.

\bibitem[Patterson \& Teh(2013)Patterson and Teh]{PatTeh2013a}
Patterson, S. and Teh, Y.~W.
\newblock Stochastic gradient {R}iemannian {L}angevin dynamics on the
  probability simplex.
\newblock In \emph{Advances in Neural Information Processing Systems}, 2013.

\bibitem[Pereyra(2013)]{pereyra2013proximal}
Pereyra, M.
\newblock Proximal {M}arkov {C}hain {M}onte {C}arlo algorithms.
\newblock \emph{arXiv preprint arXiv:1306.0187}, 2013.

\bibitem[Qi \& Minka(2002)Qi and Minka]{qiminkahessian}
Qi, Y. and Minka, T.~P.
\newblock Hessian-based {M}arkov {C}hain {M}onte {C}arlo algorithms.
\newblock In \emph{First Cape Cod Workshop on Monte Carlo Methods}, 2002.

\bibitem[Robbins \& Monro(1951)Robbins and Monro]{robbins1951}
Robbins, H. and Monro, S.
\newblock A stochastic approximation method.
\newblock \emph{Ann. Math. Statist.}, 22\penalty0 (3):\penalty0 400--407, 1951.

\bibitem[Roberts \& Stramer(2002)Roberts and Stramer]{Roberts03}
Roberts, G.~O. and Stramer, O.
\newblock {Langevin Diffusions and Metropolis-Hastings Algorithms}.
\newblock \emph{Methodology and Computing in Applied Probability}, 4\penalty0
  (4):\penalty0 337--357, December 2002.
\newblock ISSN 13875841.

\bibitem[Roberts \& Gilks(1994)Roberts and Gilks]{Roberts1994287}
Roberts, G.O. and Gilks, W.R.
\newblock Convergence of adaptive direction sampling.
\newblock \emph{Journal of Multivariate Analysis}, 49\penalty0 (2):\penalty0
  287 -- 298, 1994.

\bibitem[Rossky et~al.(1978)Rossky, Doll, and Friedman]{RosskyDollFriedman1978}
Rossky, P.~J., Doll, J.~D., and Friedman, H.~L.
\newblock {Brownian dynamics as smart Monte Carlo simulation}.
\newblock \emph{The Journal of Chemical Physics}, 69\penalty0 (10):\penalty0
  4628--4633, 1978.

\bibitem[Salakhutdinov \& Mnih(2008)Salakhutdinov and
  Mnih]{salakhutdinov2008bayesian}
Salakhutdinov, R. and Mnih, A.
\newblock Bayesian probabilistic matrix factorization using {M}arkov {C}hain
  {M}onte {C}arlo.
\newblock In \emph{International Conference on Machine learning}, pp.\
  880--887, 2008.

\bibitem[Samoradnitsky \& Taqqu(1994)Samoradnitsky and
  Taqqu]{samoradnitsky1994stable}
Samoradnitsky, G. and Taqqu, M.
\newblock \emph{Stable non-{G}aussian random processes: stochastic models with
  infinite variance}, volume~1.
\newblock CRC Press, 1994.

\bibitem[Sato \& Nakagawa(2014)Sato and Nakagawa]{icml2014c2_satoa14}
Sato, I. and Nakagawa, H.
\newblock Approximation analysis of stochastic gradient {L}angevin dynamics by
  using {F}okker-{P}lanck equation and {I}to process.
\newblock In \emph{International Conference on Machine Learning}, pp.\
  982--990, 2014.

\bibitem[Schraudolph et~al.(2007)Schraudolph, Yu, and
  G{\"u}nter]{schraudolph2007stochastic}
Schraudolph, N.~N., Yu, J., and G{\"u}nter, S.
\newblock A stochastic quasi-{N}ewton method for online convex optimization.
\newblock In \emph{International Conference on Artificial Intelligence and
  Statistics}, pp.\  436--443, 2007.

\bibitem[Shang et~al.(2015)Shang, Zhu, Leimkuhler, and
  Storkey]{shang2015covariance}
Shang, X., Zhu, Z., Leimkuhler, B., and Storkey, A.~J.
\newblock Covariance-controlled adaptive {Langevin} thermostat for large-scale
  {Bayesian} sampling.
\newblock In \emph{Advances in Neural Information Processing Systems}, pp.\
  37--45, 2015.

\bibitem[{\c{S}}im{\c{s}}ekli et~al.(2015{\natexlab{a}}){\c{S}}im{\c{s}}ekli,
  Koptagel, G{\"u}lda{\c{s}}, Cemgil, {\"O}ztoprak, and
  Birbil]{csimcsekli2015parallel}
{\c{S}}im{\c{s}}ekli, U., Koptagel, H., G{\"u}lda{\c{s}}, H., Cemgil, A.~T.,
  {\"O}ztoprak, F., and Birbil, {\c{S}}.~{\.I}.
\newblock Parallel stochastic gradient {Markov Chain Monte Carlo} for matrix
  factorisation models.
\newblock \emph{arXiv preprint arXiv:1506.01418}, 2015{\natexlab{a}}.

\bibitem[{\c{S}}im{\c{s}}ekli et~al.(2015{\natexlab{b}}){\c{S}}im{\c{s}}ekli,
  Koptagel, {\"O}ztoprak, Birbil, and Cemgil]{csimcsekli2015hamsi}
{\c{S}}im{\c{s}}ekli, U., Koptagel, H., {\"O}ztoprak, F., Birbil,
  {\c{S}}.~{\.I}., and Cemgil, A.~T.
\newblock {HAMSI}: Distributed incremental optimization algorithm using
  quadratic approximations for partially separable problems.
\newblock \emph{arXiv preprint arXiv:1509.01698}, 2015{\natexlab{b}}.

\bibitem[Stacy(1962)]{stacy1962}
Stacy, E.~W.
\newblock A generalization of the gamma distribution.
\newblock \emph{Ann. Math. Statist.}, 33\penalty0 (3):\penalty0 1187--1192, 09
  1962.

\bibitem[Teh et~al.(2016)Teh, Thi{\'e}ry, and Vollmer]{TehThiVol2014a}
Teh, Y.~W., Thi{\'e}ry, A., and Vollmer, S.
\newblock Consistency and fluctuations for stochastic gradient {L}angevin
  dynamics.
\newblock \emph{Journal of Machine Learning Research}, 17\penalty0
  (7):\penalty0 1--33, 2016.

\bibitem[Welling \& Teh(2011)Welling and Teh]{WelTeh2011a}
Welling, M. and Teh, Y.~W.
\newblock Bayesian learning via stochastic gradient {L}angevin dynamics.
\newblock In \emph{International Conference on Machine Learning}, pp.\
  681--688, 2011.

\bibitem[Xifara et~al.(2014)Xifara, Sherlock, Livingstone, Byrne, and
  Girolami]{xifara2014langevin}
Xifara, T., Sherlock, C., Livingstone, S., Byrne, S., and Girolami, M.
\newblock Langevin diffusions and the {Metropolis-adjusted Langevin} algorithm.
\newblock \emph{Statistics \& Probability Letters}, 91:\penalty0 14--19, 2014.

\bibitem[Zhang \& Sutton(2011)Zhang and Sutton]{zhang2011quasi}
Zhang, Y. and Sutton, C.~A.
\newblock {Quasi-Newton methods for Markov Chain Monte Carlo}.
\newblock In \emph{Advances in Neural Information Processing Systems}, pp.\
  2393--2401, 2011.

\end{thebibliography}
\bibliographystyle{icml2016}

\end{document}